\documentclass{article}

\usepackage{microtype}
\usepackage{graphicx}
\usepackage{subfigure}
\usepackage{booktabs} 
\usepackage{amsmath,amsfonts,amsthm}

\usepackage{hyperref}

\usepackage[accepted]{icml2024}

\usepackage{amsmath}
\usepackage{amssymb}
\usepackage{mathtools}
\usepackage{amsthm}

\usepackage[capitalize,noabbrev]{cleveref}

\theoremstyle{plain}
\newtheorem{theorem}{Theorem}[section]

\theoremstyle{definition}

\theoremstyle{remark}

\usepackage[textsize=tiny]{todonotes}

\usepackage{multirow}
\usepackage{caption}

\newcommand{\eps}{\ensuremath{\varepsilon}}

\newcommand{\mL}{\ensuremath{\mathcal L}}
\newcommand{\K}{\ensuremath{\mathcal K}}
\newcommand{\x}{\ensuremath{W}}
\newcommand{\vv}{\ensuremath{\mathbf V}}
\newcommand{\reals}{\ensuremath{\mathcal R}}

\icmltitlerunning{Submission and Formatting Instructions for ICML 2023}

\begin{document}

\twocolumn[
\icmltitle{Chain of LoRA: \\ Efficient Fine-tuning of Language Models via Residual Learning}



\icmlsetsymbol{equal}{*}

\begin{icmlauthorlist}
\icmlauthor{Wenhan Xia}{pu}
\icmlauthor{Chengwei Qin}{ntu}
\icmlauthor{Elad Hazan}{pu}

\end{icmlauthorlist}

\icmlaffiliation{pu}{Department of Computer Science, Princeton University}
\icmlaffiliation{ntu}{School of Computer Science and Engineering, Nanyang Technological University}

\icmlcorrespondingauthor{Wenhan Xia}{wxia@princeton.edu}

\icmlkeywords{Machine Learning, ICML}

\vskip 0.3in
]



\printAffiliationsAndNotice{}  

\begin{abstract}
Fine-tuning is the primary methodology for tailoring pre-trained large language models to specific tasks. As the model's scale and the diversity of tasks expand, parameter-efficient fine-tuning methods are of paramount importance. One of the most widely used family of methods is low-rank adaptation (LoRA) and its variants. LoRA encodes weight update as the product of two low-rank matrices. Despite its advantages, LoRA falls short of full-parameter fine-tuning in terms of generalization error for certain tasks. 

We introduce Chain of LoRA (COLA), an iterative optimization framework inspired by the Frank-Wolfe algorithm, to bridge the gap between LoRA and full parameter fine-tuning,  without incurring additional computational costs or memory overheads. COLA employs a residual learning procedure where it merges learned LoRA modules into the pre-trained language model parameters and re-initilize optimization for new born LoRA modules. We provide theoretical convergence guarantees as well as empirical results to validate the effectiveness of our algorithm. Across various models (OPT and llama-2) and seven benchmarking tasks, we
demonstrate that COLA can consistently outperform LoRA without additional computational or memory costs. 
\end{abstract}

\section{Introduction}

Pre-trained language models have become instrumental in natural language processing, demonstrating remarkable performance across various fields. Large language model fine-tuning is a process for adapting pre-trained models to specific tasks, allowing for improved performance on various real-world applications, such as machine translation and code analysis~\cite{lewis2019bart,wang2021codet5,qin2023improving}. Despite the notable benefits of full parameter fine-tuning, the computational expenses and memory requirements it entails present significant challenges, particularly in light of the ever-growing size of large language models.

For this reason, parameter efficient finetuning (PEFT) methods have received significant attention~\cite{pfeiffer2020adapterfusion,he2021towards}. Instead of adjusting all the parameters of the model, PEFT involves fewer adjustments to the original model parameters to specialize its knowledge for a particular application~\cite{houlsby2019parameter, prompttuning}. One of the most widely used paradigms in parameter efficient fine turning is Low-Rank Adaptation (LoRA)~\cite{hu2021lora}. LORA focuses on modifying only a small, low-rank portion of the model's weights. This is achieved by adding low-rank matrices to the weights of the model during training. The advantage of LORA is that it significantly reduces the computational burden and time required for fine-tuning, making it more efficient and scalable, especially for very large models.
Despite the significant computational advantage of LORA, it is inferior to full parameter fine-tuning in terms of generalization error.

In this paper we investigate whether the generalization error gap between LORA and full parameter fine-tuning can be reduced albeit preserving the computational efficiency. We do this by learning a higher rank augmentation of the LLM weights by method of residual learning. The high rank augmentation is composed of several low rank structures. Namely, we use
an iterative procedure to learn a low-rank addition to
the existing approximation, thereby increasing its rank. Hence, we call the procedure ``chain of LORA", or COLA for short. 

This residual learning method is inspired by the Frank-Wolfe algorithm as applied to  matrix  completion, which augments an existing completion by a rank one addition. Over many iterations, this residual learning procedure can be shown to produce an accurate higher rank completion.   

\textbf{Our contributions}

We present an iterative optimization framework, COLA, for parameter efficient fine tuning. 
COLA is based on the Frank Wolfe method from mathematical optimization, and we formalize this relationship.

We demonstrate the effectiveness of COLA via extensive experiments across datasets and models. COLA consistently outperforms LoRA in terms of generalization error with no additional cost of compute. For example, fine-tuning OPT-1.3B with COLA brings a relative $6.47\%$ test accuracy gain to LoRA on WSC. LLama2-7B experiments shows up to $4.4\%$ relative test score improvement.

We provide theoretical analysis of the iterative learning framework employed in our proposed method, demonstrating the  convergence to stationary points in the setting of smooth nonconvex optimization.

\section{Related Work}
Conventional full parameter fine-tuning becomes computationally impractical as both model size and the number of downstream tasks increase. In response to this challenge, recent advancements in parameter-efficient finetuning methods suggest modifying only a small portion of parameters while maintaining the majority of pre-trained model parameters unchanged.

\paragraph{Adaper based methods.}
Within this domain, a line of research known as adapter based approach involves inserting compact adapter modules between transformer layers. Throughout the fine-tuning process, only the newly introduced lightweight adapters are trained, while the pre-trained model remains frozen and shared across tasks, thus significantly enhancing the practicality and efficiency of adapting large models to diverse tasks. \citet{houlsby2019parameter} propose a new bottleneck adapter module and position it twice within each transformer~\cite{vaswani2017attention} layer. The adapter employs a bottleneck architecture, incorporating a skip connection to effectively constrain the number of parameters involved in the module design. Variant adapter architecture and placements are proposed in concurrent work \cite{bapna-firat-2019-simple,pmlr-v97-stickland19a}. Building upon the success of adapter-based approaches for single-task adaptation, subsequent studies extend the adapter-based architecture to the realm of multi-task learning scenarios~\cite{mahabadi2021parameter}. AdapterFusion proposes a two-stage learning framework where task-specific adapters are learned and then later combined in a separate knowledge composition step ~\cite{pfeiffer2020adapterfusion}.

\paragraph{Prefix tuning methods.}
Alternative research investigates the incorporation of tunable parameters into both the input and hidden layers, as explored by ~\citet{li2021prefix}. These lightweight task-specific vectors, commonly referred to as the prefix, offer a notable reduction in the memory load required for storing task-specific models. Additionally, they outperform full fine-tuning, particularly in scenarios with limited data availability. Efficient prompt tuning further simplifies prefix tuning by concatenating a trainable tensor ("soft prompt") with the model's input embeddings ~\cite{prompttuning}. These "soft prompts" are learned through backpropagation to perform downstream tasks. 

\paragraph{LoRA and its variants.}
The most closely related work to ours is LoRA~\cite{hu2021lora}, which introduces trainable low-rank matrices to approximate weight update during fine-tuning. We elaborate on its technical details in the preliminaries section below. Building upon the foundation laid by LoRA, numerous recent studies have explored its variants from different perspectives. QLoRA~\cite{dettmers2023qlora} further leverages 4-bit quantization to effectively and efficiently fine-tune LLMs. To enhance parameter efficiency, Tied-LoRA, introduced by ~\citet{renduchintala2023tied}, incorporates weight tying and selective training. ~\citet{chen2023longlora} propose LongLoRA to extend the context sizes of LLMs with limited computation cost. MultiLoRA~\cite{wang2023multilora} is designed specifically for better multi-task adaptation. Concurrently,~\citet{sheng2023s} introduce S-LoRA, offering a framework that enhances the scalable serving of multiple LoRA adapters.

Optimization for fine tuning of LLM has special challenges, notably memory constraints. For this reason,  zero-order optimization methods were proposed \cite{mezo}.


\section{Our Method}
\label{sec:method}
In this section we describe our method for fine tuning. It is divided into two parts, in the first we present necessary background for our exposition, and the second gives details of COLA. 

\begin{figure*}[h]
  \begin{center}
   \includegraphics[width=1.8\columnwidth]{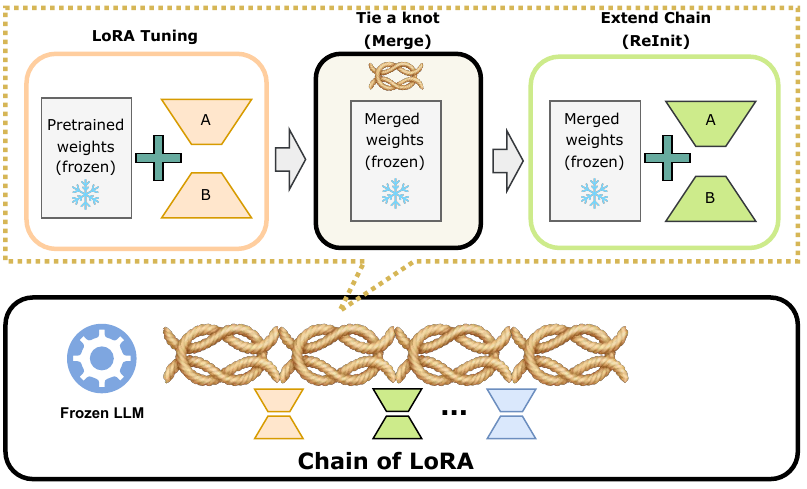}
  \end{center}
  \caption{An illustration of Chain of LoRA. Our approach starts with a frozen LLM, and learns a sequence of low-rank matrices to approximate a high-rank augmentation to perform task adaptation. As shown in the dashed line box, each residual learning procedure consists of three steps: (1) LoRA Tuning, (2) Tie a knot, and (3) Extend the chain. In step 1, low-rank LoRA modules are fine-tuned, In step 2, the learned LoRA weights are merged into the frozen model. In step 3, a new LoRA module is instantiated and the optimizer state is reset. These three steps are repeated in this residual learning paradigm.}
  \label{fig:COLA}
\end{figure*}

\subsection{Preliminaries}
\textbf{Low Rank Adaptation (LoRA)} ~\cite{hu2021lora} aims to improve the efficiency of fine-tuning large language models by training much smaller low-rank decomposition matrices of certain weights. It hypothesizes a low "intrinsic rank" of weight updates at task adaptation and injects trainable low-rank decomposition matrices into each layer of the Transformer architecture. Consider a weight matrix $W_{frozen}$ from the pre-trained model, the weight update $\Delta W$ for task adaptation is represented with a low-rank decomposition $BA$. The forward pass with LoRA is as follows:
    $$W_{frozen}x + \Delta Wx = W_{frozen}x + \mathbf{BA}x, $$
where $W_{frozen}, \Delta W \in \mathbb{R}^{d\times k}$, $A \in \mathbb{R}^{r\times k}$, $B \in \mathbb{R}^{d\times r}$ and $r \ll min(d, k)$. $A$ is typically initialized with random Gaussian initialization and $B$ is initialized with zero to have $\Delta W = 0$ at the start of training. During training, $W_{frozen}$ is frozen and only $B$, $A$ are optimized. At deployment, the learned low-rank matrices can be merged with the frozen weights of the pre-trained model. 

\textbf{Frank-Wolfe}
The Frank-Wolfe method, also known as the conditional gradient method, is an optimization algorithm for solving constrained convex, and more recently nonconvex, optimization problems. The key feature of the Frank-Wolfe method is how it handles the constraints. Instead of projecting onto the constraint set via projections, it uses a linear optimization oracle. Iteratively, the method finds a linear approximation of the objective function within the feasible region and moves towards the minimizer of this approximation.  

The Frank-Wolfe algorithm is particularly suited for problems in which linear optimization is easier than Euclidean projections. For this reason, ``projection free" methods were considered in the machine learning community \cite{Hazan08,Jaggi13b,HazanK12,garber2016linearly}.  More recently nonconvex optimization was considered using the Frank Wolfe method in \cite{lacoste2016convergence,reddi2016stochastic}.

\subsection{Chain of LoRA}\label{subsec:cola}
In this section we give the details of our simple yet effective optimization framework for efficient parameter finetuning of large language models. The key idea of our method is to form a chain of LoRAs and iteratively learn the low-rank adaptation LoRA modules. As illustrated in Figure~\ref{fig:COLA}, our method is comprised of three stages: \textbf{Tune LoRA, Tie a knot, Extend the chain}. We first introduce  notations, followed by an explanation of the three stages in the workflow. We also provide the detailed step-by-step procedure in Algorithm ~\ref{pseudo}.

\begin{algorithm}[h]
\caption{Chain of LoRA (COLA)}
\label{pseudo}
\begin{algorithmic}
\STATE Input: frozen pre-trained weights $W$, chain knots $\{\tau_1,\dots, \tau_m\}$, finetuning dataset $\mathcal{D}$, training objective $\mathcal{L}$, total training iterations T.
\STATE Initialize LoRA params to $A_0, B_0$
\FOR{$t = 1, \ldots, T$}
\STATE Sample minibatch $\mathcal{B}_t \subset \mathcal{D}$
\IF{$t \in \{\tau_1,\dots, \tau_m\}$} 
    \STATE \textbf{Tie knot}: Merge LoRA to backbone weights $W = W + B_tA_t$ \\

    \STATE \textbf{Extend chain}: 
    Re-initialize LoRA parameters $A_t = A_0, B_t = B_0$
\ENDIF

\STATE forward pass with LoRA
\STATE backward pass and update LoRA parameters 
$$ (A_t,B_t) = (A_{t-1},B_{t-1}) - \eta_t* \hat{\nabla}_{A,B} \mathcal{L}(W_t) $$

\ENDFOR
\vskip -0.3in
\end{algorithmic}
\end{algorithm}

For a pre-trained LLM weight matrix $W_{pretrained} \in \mathbb{R}^{d\times k}$, we denote the weights update occurred during fine-tuning as $\Delta W$. Ideal adaptation yields the optimal weights $ W^\star$ tailored for the given task and the corresponding optimal weight update $\Delta W^\star$, as shown below. 
$$W^\star = W_{pretrained} + \Delta W^\star$$

In COLA, we propose to approximate $\Delta W^\star$ with a chain (basically a sequence) of low-rank matrix decompositions $\{(A_1,B_1), \dots, (A_M,B_M)\}$, where $A_i \in \mathbb{R}^{r_i\times k}$, $B_i \in \mathbb{R}^{d\times r_i}$ and $r_i \ll min(d, k)$ for $1 \leq i \leq M$. Each low-rank tuple $(A_i,B_i)$ is obtained by optimizing
$$\arg \min_{B_iA_i} \mathcal{L}(W_{pretrained} + \sum_{j=1}^{i} B_jA_j) ,$$ 
where $\mathcal{L}$ is the task-specific objective function. \textbf{\textit{COLA follows an iterative residual learning paradigm.}} Fine-tuning each $(A_i,B_i)$ can be viewed as learning the residual of $\Delta W^\star - \sum_{j=1}^{i-1} B_jA_j$, which is an easier optimization problem compared to learning $\Delta W^\star$ from scratch.
We hypothesize that $\sum_{i=1}^{M} B_iA_i$ approximates $\Delta W^\star$ better than a single LoRA update $BA$, and we design a chaining framework to achieve this with less computation compared to the baseline LoRA. 

COLA forms a chain of LoRAs by iteratively tuning, merging, and extending LoRA modules, as depicted in Figure~\ref{fig:COLA}. We denote the length of the chain in COLA as the number of residual LoRA modules optimized. For COLA with a chain length of M, the three sub-steps in Figure~\ref{fig:COLA} are repeated M times. Below we describe the three  sub-steps in detail.

\textbf{Tune LoRA}: In this step, we perform standard LoRA tuning, i.e., learning only the A and B matrices and leaving all other model parameters untouched. At initialization of COLA, this step learns LoRA modules $(A_1, B_1)$ on top of the frozen pre-trained LLM weights $W_{pretrained}$. After the initial phase of COLA, the LoRA modules $(A_i, B_i)$ are fine-tuned on top of fixed model weights incorporated with previously learned LoRAs' weights. The fixed model weights at the i-th iteration of COLA is $W_{pretrained} + \sum_{j=1}^{i-1} B_jA_j$.

\textbf{Tie a knot}: After the current LoRA modules $(A_i, B_i)$ are trained, we merge them into the previously frozen LLM weights and we refer to this step as "tie a knot". This way, we incorporate the weight update, approximated by $B_iA_i$, into the frozen model weights. The resulting frozen model weights becomes $W_{pretrained} + \sum_{j=1}^{i} B_jA_j$. This allows learning only the residual information $\Delta W^\star - \sum_{j=1}^{i} B_jA_j$ for the next iteration. Additionally, merging the LoRA modules into the frozen LLM helps reduce memory burden under limited resource scenarios. Instead of storing a list of LoRA modules introduced in the COLA, merging them to the frozen model weights in a running fashion helps keep the GPU memory consumption the same as training LoRA only once.

\textbf{Extend the chain}: We extend the COLA chain by re-initializing a new set of LoRA module $(A_{i+1}, B_{i+1})$ to learn the residual weights update needed to adapt the LLM to certain task. In this step, the newly introduced $A_{i+1}$ adopts Gaussian initialization and $B_{i+1}$ is initialized to zero, following ~\citet{hu2021lora}. Additionally, we reset all of the optimizer states, including but not limited to the parameters to be optimized and the gradient history. 

\paragraph{Training and Inference cost of COLA.}
The training cost of COLA is determined by the rank of the LoRA modules used to form the chain. The training computation for COLA is the same as LoRA when the rank is the same. In COLA, progressively lowering the rank of the LoRA modules may be an effective strategy to approximate optimal residual weight updates for specific tasks and lower the overall training cost. We explore this direction in our experiment section. At inference, all of the learned $B_jA_j$ can be integrated into the original model weights. Since $W_\text{pretrained}$ has the same shape as $B_jA_j$, the final integrated model weight has the same number of parameters as the original pre-trained LLM. Therefore, no latency overhead is introduced during inference.

\section{Convergence of COLA and the Nonconvex Frank-Wolfe method}

The COLA algorithm described in figure \ref{pseudo} is motivated by and closely related to the Frank Wolfe algorithm \cite{frank1956algorithm}. To see this, notice that COLA is an iterative algorithm whose iterations are succinctly described by the equation
$$  W \leftarrow W + \arg \min_{BA} \mathcal{L}(W + BA)  . $$
Taking the linear Taylor approximation we can write 
$$ \mathcal{L}(W + BA ) \approx L(W) + \nabla \mathcal{L}(W) \times BA , $$
and thus, a constrained minimization over a set $\K \subseteq \reals^d$ can be seen to be approximately 
$$ \arg \min_{BA \in \K} \mathcal{L}(W + BA)  \approx \arg\min_{BA \in \K}  \nabla \mathcal{L}(W) \times BA  .$$
This is reminiscent of the Frank-Wolfe algorithm, which was historically developed in the context of linear programming. Below we analyze a variant of the Frank Wolfe algorithm for stochastic non-convex smooth optimization. The algorithm pseudo-code is given in Algorithm \ref{alg:condgrad}, and it is written in COLA notations as an application to fine tuning of LLM. The stochasticity is captured in equation \eqref{algstep:linearopt}, where it is assumed that the direction of the gradient is approximated up to $\eps$ using a stochastic gradient method.

\begin{algorithm}[H]
\caption{Idealized COLA}
\label{alg:condgrad}
\begin{algorithmic}
\STATE Input: step sizes $\{ \eta_t \in (0,1] , \ t \in [T]\}$, initial $W_1 \in \K$. 
\FOR{$t = 1$ to $T$}
\STATE Approximate via stochastic optimization 
\begin{equation}  \label{algstep:linearopt} \vv_{t} \in_\eps  \arg \min_{\x \in \K} \left\{\x^\top {\nabla} \mL(W_t) \right\} 
\end{equation}
\STATE $W_{t+1} \gets W_t + \eta_t(\vv_t - W_t)$.
\ENDFOR
\end{algorithmic}
\end{algorithm}

Specifically, we assume that COLA performs gradient updates such that after every epoch we have that 
\begin{equation*}
\vv_{t}^\top {\nabla} \mL(W_t) \leq   \arg \min_{\x \in \K} \left\{\x^\top {\nabla} \mL(W_t) \right\} + \eps.
\end{equation*}
Notice that we have replaced the low rank matrices $A,B$ with a single matrix $W$. This deviates from the exact specification of COLA, but can be justified according to the following intuition. Linear optimization over the trace norm ball results in a rank one solution, as shown in the context of the Frank Wolfe method in \cite{Hazan08,allen2017linear}. In COLA, we perform nonconvex optimization over $A,B$ directly, and their rank can be larger than one. 

Below we give an analysis of this algorithm which incorporates the stochastic approximation of the iterates $A_t,B_t$. 

Henceforth, let $h_t = \mL(\x_t) - \mL(\x^*) $, and 
$$ g_t \triangleq \left\{ \max_{\vv \in \K} \nabla \mL(\x_t)^\top (\vv - \x_t) \right\} . $$
The latter quantity is a metric of convergence in nonconvex optimization, which is sometimes called the Frank-Wolfe gap. Notice that $g_t$ is zero if and only if the projected gradient of $\mL$ at $\x_t$ is zero. 

The following theorem establishes that Algorithm \ref{alg:condgrad} guarantees average duality gap approaching zero for stochastic smooth nonconvex optimization, as long as the distribution shift is bounded sublinearly with time.
\begin{theorem} \label{thm:offlineFW}
Algorithm \ref{alg:condgrad} applied to a sequence of stochastic gradients of $\beta$-smooth nonconvex functions that are bounded in $\K$ by $M$, with step sizes $\eta_t =  \frac{\sqrt{M}}{D \sqrt{\beta T}}$  attains the following convergence guarantee
$$ \frac{1}{T} \sum_{t=1}^T g_t \leq \frac{2 \sqrt{ M \beta}  D }{\sqrt{T}} + \eps$$
\end{theorem}
\begin{proof}
We denote $\nabla_t = \nabla \mL(\x_t)$.  For any set of step sizes, we have
\begin{eqnarray*}\label{old_fw_anal}
	& h_{t+1}  =   \mL(\x_{t+1}) - \mL(\x^\star)  \\
	& =   \mL(\x_t + \eta_t(\vv_t - \x_t)) - \mL(\x^\star)   \\
	&\leq \mL (\x_t) - \mL (\x^\star) + \eta_t(\vv_t-\x_t)^{\top}\nabla_t \\
 & + \eta_t^2 \frac{\beta}{2}\Vert{\vv_t-\x_t}\Vert^2   &  \textrm{smoothness }  \\
 &\leq  \mL(\x_t) - \mL(\x^\star) +   \eta_t(\vv_t-\x_t)^{\top} \nabla_t   \\
 & + \eta_t^2 \frac{\beta}{2} D^2   \\
	&\leq  h_t + \eta_t ( g_t + \eps) + \eta_t^2 \frac{\beta D^2 }{2} . &  \textrm{$\vv_t$ choice} . 
\end{eqnarray*}
Here we denoted by $D$ the diameter of the set $\K$. 
We reached the equation $  g_t + \eps \leq \frac{h_t - h_{t+1}}{\eta_t} +  \eta_t \frac{\beta D^2} {2}  $. Summing up over all iterations and normalizing we get ,

\begin{eqnarray*}
\frac{1}{T} \sum_{t=1}^T g_{t}  + \eps & \leq \frac{h_0 - h_T} {\eta T} + \eta  \beta D^2 \\
& \leq \frac{M}{\eta T } + \eta \beta D^2  \\
& \leq \frac{2 \sqrt{M \beta} D }{\sqrt{T} },     
\end{eqnarray*}
which implies the Theorem.
\end{proof}

\section{Experimental Setup}

\begin{table*}[t]
\centering
\small
\setlength\tabcolsep{3pt}
\resizebox{0.75\linewidth}{!}{
\begin{tabular}{l|c|c|c|c|c|c|c}
\toprule
\multirow{1}{*}{Task} & \multicolumn{1}{c}{SST-2} & \multicolumn{1}{|c}{WSC}& \multicolumn{1}{|c}{CB}& \multicolumn{1}{|c}{WIC} & \multicolumn{1}{|c}{BoolQ} & \multicolumn{1}{|c}{MultiRC} & \multicolumn{1}{|c}{RTE} \\
\midrule
LoRA& 93.16 & 56.53 & 75.35 & 63.47 & 70.70 &  68.94 & 72.49\\
COLA (ours)& \textbf{93.32}& \textbf{60.19 }& \textbf{76.42}& \textbf{64.26}& \textbf{72.08 }& \textbf{70.63}& \textbf{74.15}\\
relative gains &0.17$\%$ &6.47$\%$ &1.42$\%$ &1.24$\%$ &1.95$\%$ &2.45$\%$ &2.29$\%$  \\
\bottomrule
\end{tabular}
}
\caption{Experiments on OPT-1.3B with 1,000 test examples over various tasks. Task performance is reported after averaging over five random seeds. COLA consistently outperforms LoRA across all tasks.}
\vspace{-0.5em}
\label{table:opt1.3b}
\end{table*}

In this section, we initially outline the tasks and models, followed by an introduction to the methods under comparison in our study. Finally, we provide details on the implementation.

\subsection{Models and Tasks}

\textbf{models}:
We experiment with COLA to fine-tune OPT-1.3B~\cite{zhang2022opt} and Llama2-7B ~\cite{llama2}. Both models' pre-trained checkpoints are from HuggingFace. 

\textbf{datasets}:
We evaluate the effectiveness of our method and compare it with the LoRA baseline on task adaptation across seven classification tasks: SST-2, WSC, CB, WIC, BoolQ, MultiRC, and RTE.  

\textbf{methods compared}:
For the current writeup, we mainly compare with LoRA, a representative PEFT method to train only low-rank matrices while keeping the pre-trained model parameters frozen. For future work, we will add in more baselines. 

\subsection{implementation details}
We implemented our method with the PyTorch and Transformers library~\cite{wolf2020transformers}. All experiments are carried out on NVIDIA A100 (80G) GPU. 

We adopt the experimental configuration outlined in ~\citet{mezo}, where we randomly select 1000 examples for training, 500 for validation, and another 1000 for testing across each dataset under consideration. In COLA raining, we use AdamW ~\cite{adamw} as the base optimizer and train for a total of 5 epochs. For a fair comparison, we keep the epoch number consistent with our baseline. A linear learning rate schedule is applied with the initial learning rate selected from $\{1\times 10^{-3}, 8\times 10^{-4}, 5\times 10^{-4}, 1\times 10^{-4}, 5\times 10^{-5}\}$. The batch size is chosen from $\{4, 8\}$. The reported results represent the best score after hyperparameter grid-search for all experiments, conducted over five random seeds. 

In implementing LoRA, we adhere to the practice outlined in \citet{hu2021lora}, introducing trainable linear low-rank modules to both query and value projections within all self-attention layers. 
While some research ~\cite{zhang2023adaptive} has explored the application of LoRA to all projection matrices or all weight matrices, the specific choice of where to apply LoRA is not a pivotal aspect of our work.
For OPT experiments, we incorporate bias into the injected LoRA modules, aligning with the approach taken in~\citet{mahabadi2021parameter}. Conversely, in Llama-2 experiments, we deliberately disable bias in LoRA to ensure module key matching with the pre-trained checkpoint "meta-llama/Llama-2-7b-hf." In all experiments, we set the rank of LoRA (denoted as "r") to 8 and $\alpha$ to 16, where the ratio $\alpha/r$ is employed to scale the weight updates.


\section{Results and analysis}

\subsection{Main Results}\label{subsec:mainresults}
We report the test performance of our method and baseline across various tasks in this section. The experiment results on OPT-1.3B are detailed in ~\Cref{table:opt1.3b}, and the results for Llama2-7B are provided in ~\Cref{table:llama2-7B}. Notably, our method consistently outperforms LoRA on all datasets under the same training budget, showcasing its superior performance.

Specifically, for OPT-1.3B experiments, COLA brings a performance boost to LoRA by 3.66 (relative improvement of 6.47$\%$), 1.38 (relative improvement of 1.95$\%$), 1.66 (relative improvement of 2.29 $\%$ ) on tasks WSC, BoolQ and RTE, respectively. 

For Llama2-7B experiments, COLA boosts the test score on RTE from 82.09 to 85.70, which corresponds to a 3.61 gain and 4.40$\%$ relative improvement. On SST-2, the average test scores for both our method and the baseline are the same, possibly due to the relatively low task complexity and the utilization of a subset of test examples.

In our reported results, as detailed in ~\Cref{table:opt1.3b} and ~\Cref{table:llama2-7B}, we maintain consistency by setting the rank of all injected modules in the sequence to 8, aligning with the baseline LoRA setup. Additionally, we use an equal training epoch budget for different methods and thus ensuring the same training computation cost, as explained earlier in Section~\ref{sec:method} .

\begin{table}
\centering
\small
\setlength\tabcolsep{3pt}
\begin{tabular}{l|c|c|c|c|c}
\toprule
\multirow{1}{*}{Task} & \multicolumn{1}{c}{SST-2} & \multicolumn{1}{|c}{WSC}& \multicolumn{1}{|c}{CB}& \multicolumn{1}{|c}{WIC} & \multicolumn{1}{|c}{RTE} \\
\midrule
LoRA& 95.82 & 57.30 & 91.78 & 71.59 & 82.09 \\
COLA (ours)& \textbf{95.82}& \textbf{59.80}& \textbf{93.21}& \textbf{71.66}& \textbf{85.70} \\
relative improvement &0$\%$ &4.36$\%$ &1.56$\%$ & 0.10$\%$& 4.40$\%$ \\

\bottomrule
\end{tabular}
\caption{Experiments on Llama2-7B with 1,000 test examples over various tasks. Task performance is reported after averaging over five random seeds. COLA consistently outperform LoRA across all tasks.}
\vspace{-0.5em}
\label{table:llama2-7B}
\end{table}

\begin{table*}[tp]
  \centering
  \begin{tabular}{@{}lcccccc@{}}
    \toprule
    & \multicolumn{2}{c}{\textbf{CB}} & \multicolumn{2}{c}{\textbf{WSC}} & \multicolumn{2}{c}{\textbf{WIC}} \\
    \cmidrule(r){2-3} \cmidrule(lr){4-5} \cmidrule(l){6-7}
    Methods & test score & train FLOPs saved & test score & train FLOPs saved & test score & train FLOPs saved\\
    \midrule
    LoRA & 75.35 & - & 56.53 & - & 63.47 & - \\
    \midrule
    COLA (8, 8) &\textbf{76.78} & - &\textbf{59.81} & - &63.51 &\\
    COLA (8, 6) &76.43 &3.60$\times10^{11}$ &58.26 &4.28$\times10^{10}$ &63.85 &5.21$\times10^{11}$\\
    COLA (8, 4) &75.35 &7.20$\times10^{11}$ &57.30 &8.56$\times10^{10}$ &\textbf{64.04} &1.04$\times10^{12}$\\
    COLA (8, 2) &76.07 &1.08$\times10^{12}$ &57.30 &1.28$\times10^{11}$ &63.19 &1.56$\times10^{12}$\\

    \bottomrule
  \end{tabular}
\caption{COLA rank step-down experiments. Method COLA ($r_1$, $r_2$) indicates that the first iteration learns LoRAs with rank $r_1$, and the second iteration learns LoRAs with rank $r_2$. All numbers are reported over five random seeds. COLA (8,8) uses the same amount of training FLOPs as the baseline, as denoted by ``-''.}
\vspace{-0.5em}
\label{table:rankdecay}
\end{table*}

\subsection{Ablation Study}
\textbf{Different number of LoRAs in the chain}: 
As described in Section~\ref{subsec:cola}, COLA consists of repeated iterations of LoRA tuning and merging. We denote the length of COLA as the number of LoRAs learned and merged in the fine-tuning process. To investigate the effect of the chain length of COLA on task adaptation performance, we further conduct experiments by varying the number of LoRAs. Specifically, we studied chain length of 1, 2, 3 and present the findings in Figure~\ref{fig:chainnumber}. 

Here, chain length of 1 corresponds to the baseline LoRA fine-tuning. All experiments are conducted with a total of 5 training epochs. For example, in COLA experiments with chain length of 2, the first LoRA training phase lasts from epoch 1 to epoch 3. After the first LoRA module merges with the pretrained LLM weights and optimizer states reinitialize, the second LoRA starts from epoch 4 to epoch 5, which in total uses the same 5 total training epochs. All experiments results are reported over five random seeds. 

As shown in Figure~\ref{fig:chainnumber}, there is a growing trend of test accuracy as the chain length increases across tasks. This is consistent with our hypothesis that residual learning of LoRA modules will lead to a better approximation of the optimal weight update to the fixed pre-trained LLM for task adaptation. For a majority of tasks, COLA is more robust in terms of generalization error compared to baseline LoRA, as shown by COLA's smaller standard deviations. 

\begin{table}
\centering
\small
\setlength\tabcolsep{3pt}
\begin{tabular}{l|c c c}
\toprule
& \multicolumn{1}{c}{\textbf{WSC}} & \multicolumn{1}{c}{\textbf{WIC}}& \multicolumn{1}{c}{\textbf{MultiRC}}\\
\midrule
length = 1 &56.53 ($\pm$ 7.67) &63.47 ($\pm$ 1.34) &68.94 ($\pm$ 3.06) \\
length = 2 & 59.81 ($\pm$ 4.10) &63.51 ($\pm$ 1.25) &69.44 ($\pm$ 1.55)\\
length = 3 &\textbf{60.19 ($\pm$ 3.77)} & \textbf{64.26 ($\pm$ 2.49)} &\textbf{70.63 ($\pm$ 2.12)}\\
\midrule
& \multicolumn{1}{c}{\textbf{CB}} & \multicolumn{1}{c}{\textbf{BoolQ}}& \multicolumn{1}{c}{\textbf{RTE}}\\
\midrule
length = 1 &75.35 ($\pm$ 4.84) &70.70 ($\pm$ 2.04) &72.49 ($\pm$ 2.39)\\
length = 2 &\textbf{76.78 ($\pm$ 6.38)} &72.04 ($\pm$ 1.75) &72.63 ($\pm$ 1.46)\\
length = 3 &76.42 ($\pm$ 4.97) &\textbf{72.08 ($\pm$ 1.27)} &\textbf{74.15 ($\pm$ 1.36)}\\
\bottomrule
\end{tabular}
\caption{COLA evaluation with varying chain length. Test score across tasks is reported using 5 random seeds and is presented in the ``average ($\pm$ standard deviation)'' format. The highest average performance for each task is highlighted in bold.}
\vspace{-0.5em}
\label{table:chainlength}
\end{table}

\begin{figure}[h]
  \begin{center}
\includegraphics[width=1
\columnwidth]{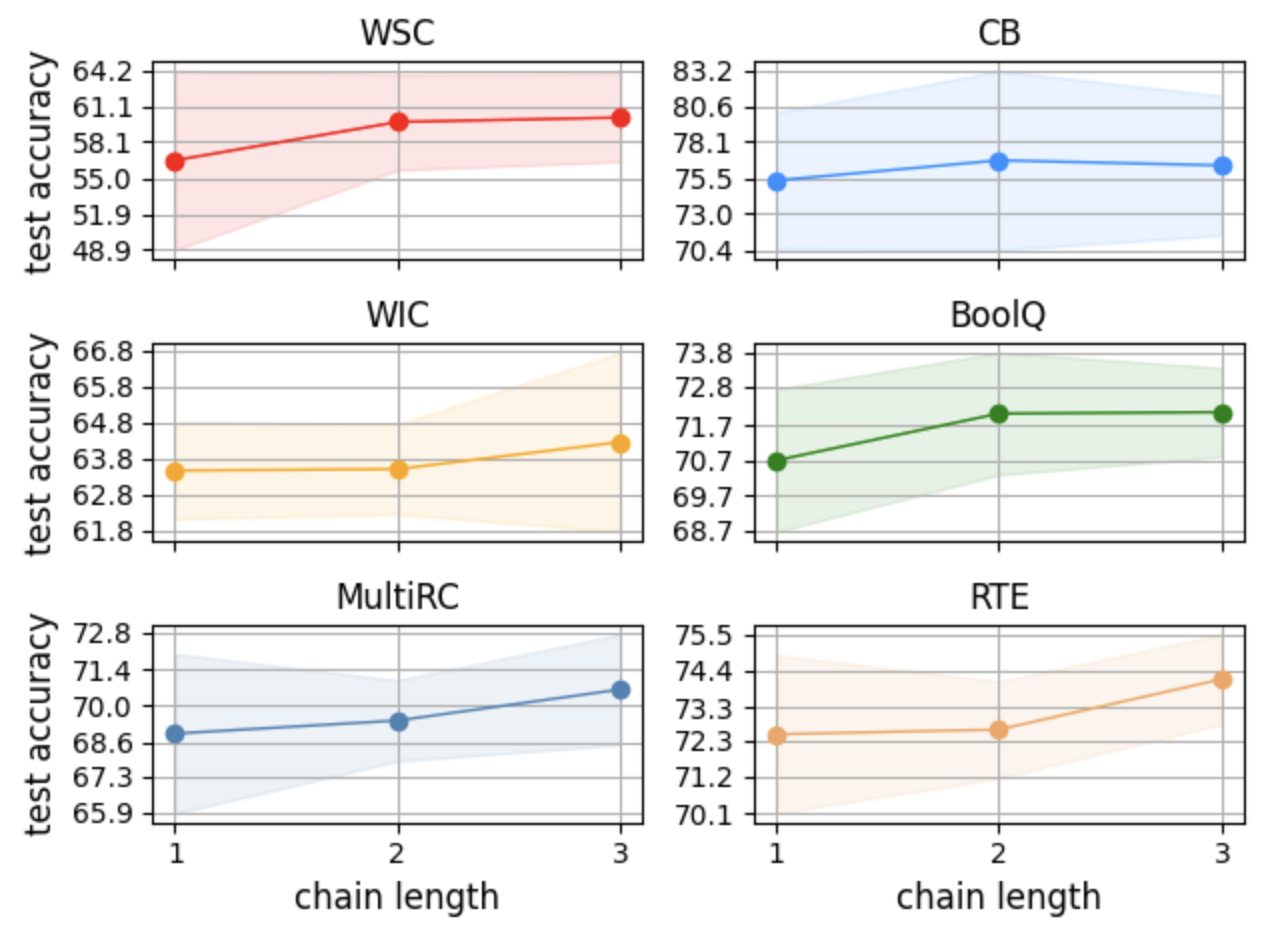}
  \end{center}
  \caption{Ablation study on test accuracy with different COLA chain length across tasks. Results are reported after averaging five different seeds and the shaded area corresponds to standard deviation. The general trend is that the test accuracy increases with the chain length.}
  \label{fig:chainnumber}
\end{figure}

\textbf{Rank step-down}:
Since COLA is a residual learning paradigm, we hypothesize that the residual weight update to be learned for task adaptation should be progressively lower in rank. Therefore, instead of using a chain of LoRAs with a fixed rank of eight, as described in Section \ref{subsec:mainresults}, we conduct further studies on lowering the rank.

Here, we consider a simple setting of COLA with length of two. We fix the rank to 8 for the first three epochs and set the rank for the remaining epochs to either 2, 4, 6, or 8. We show the results in Figure~\ref{fig:rankdecay} and report the test performance in Table~\ref{table:rankdecay}.

Figure~\ref{fig:rankdecay} show that COLA with rank step-down outperforms LoRA with a fixed rank of 8 for all tasks (with the exception of one data point--WIC with rank 2). Thus COLA with rank step-down offers both superior generalization ability over standard LoRA and lower computational cost. In addition, our results indicate that the optimal rank to use for COLA is task-dependent. The CB and WSC tasks both benefit from higher rank LoRA modules in the second learning phase. The WIC task, on the other hand, surprisingly shows maximal test accuracy at a rank of 4 for ($A_2$, $B_2$). 

\begin{figure}[h]
  \begin{center}
\includegraphics[width=1
\columnwidth]{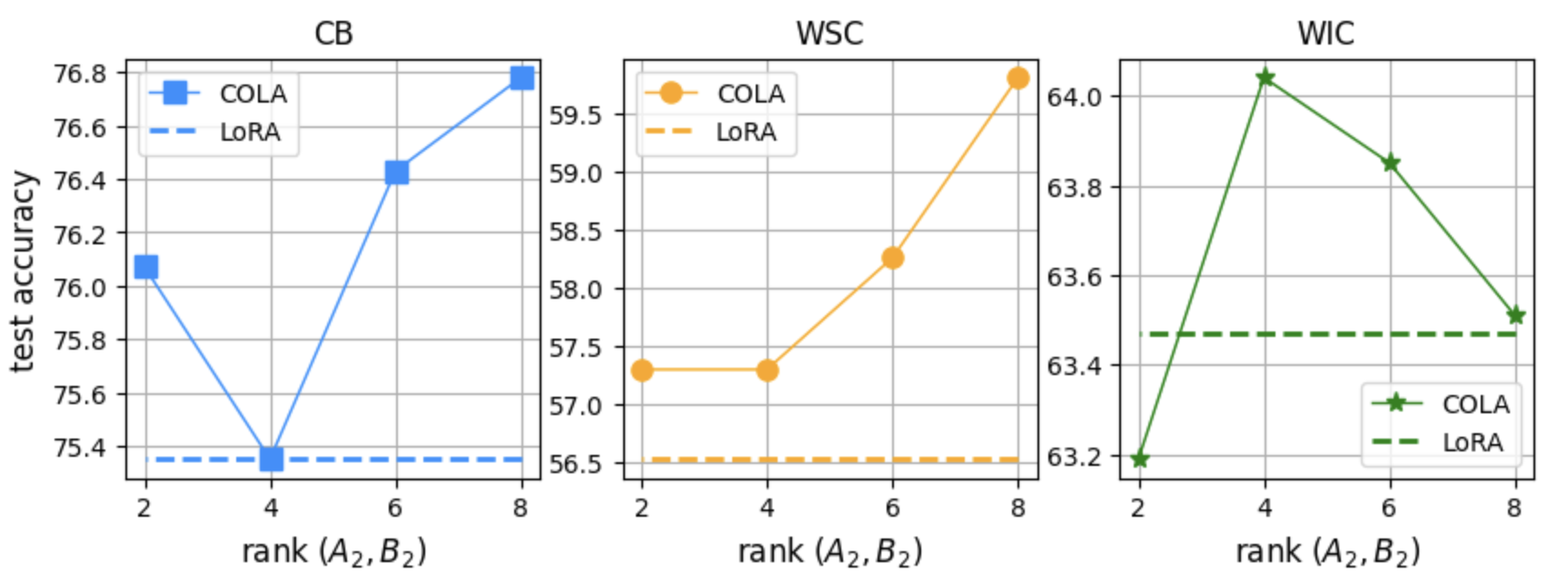}
  \end{center}
  \caption{COLA with rank step-down for three tasks. Experiments are conducted with COLA of length 2 where ($A_1$, $B_1$) has fixed rank of 8, and ($A_2$, $B_2)$ rank is as shown in the figure.}
  \label{fig:rankdecay}
\end{figure}

\textbf{Computation comparison}
Table~\ref{table:rankdecay} provides a detailed comparison of the training computation cost between COLA of different rank step-down configurations and the baseline.

The training FLOPs are obtained from the HuggingFace trainer state, and are reported as the aggregate over five random seeds. The baseline LoRA uses a fixed rank of 8 throughout training, while COLA starts with rank 8 and continues with different ranks in the residual learning phase. As expected, stepping down the rank in the chain results in higher FLOPs savings. Overall, COLA offers lower generalization error with less compute.

\section{Conclusions and future work}
In this work, we introduce Chain of LoRA (COLA) for efficient fine-tuning of large language models. The idea is to use an iterative low rank residual learning procedure to approximate the optimal weight update needed for task adaptation. Preliminary experimental results show that COLA consistently outperforms previous baselines albeit using the same, or less, computational resources. 

We are actively working on applying COLA with different base optimizers and further experiments on larger scale LLMs. We are also experimenting beyond classification tasks, namely generation, summarization, and multiple choice. 

\appendix
\section{Appendix}
\label{sec:appendix}

\subsection{Case Study} \label{sec:case_study}




\label{sec:case_study}

\bibliography{COLA}
\bibliographystyle{icml2024}


\end{document}